\documentclass{article}
\usepackage[dblblindworkshop,final]{neurips_2026}
\usepackage[utf8]{inputenc}
\usepackage[T1]{fontenc}
\usepackage{microtype}
\usepackage{amsmath,amssymb,amsthm,mathtools}
\usepackage{booktabs,multirow,tabularx}
\usepackage{graphicx,xcolor}
\usepackage{tikz}
\usetikzlibrary{arrows.meta,positioning,fit}
\usepackage{hyperref}
\usepackage[font=small,labelfont=bf]{caption}
\usepackage{subcaption}
\usepackage{enumitem}
\setlist{nosep,leftmargin=*}
\hypersetup{colorlinks=true,citecolor=blue!55!black,linkcolor=blue!55!black,urlcolor=blue!55!black}
\newtheorem{theorem}{Theorem}
\newtheorem{corollary}{Corollary}
\newtheorem{proposition}{Proposition}
\theoremstyle{remark}
\newtheorem{remark}{Remark}
\newcommand{\method}{\textsc{Cask}}
\newcommand{\caskr}{\textsc{Cask}@R95}
\newcommand{\vech}{\operatorname{vech}}

\title{\textbf{Whose Memory Is It? Scope-Aware Commit Rules for Long-Term LLM Memory}}
\workshoptitle{Linguistic Principles for Foundation Models (LP4FM)}

\author{%
	\parbox[t]{0.40\textwidth}{\centering\normalfont
		\textbf{Hongyu Gu}\textsuperscript{1}\\
		University of Science\\
		and Technology of China\\
		Hefei, China\\[0.35em]
		{\scriptsize\texttt{ustc\_23ghy@mail.ustc.edu.cn}}}%
	\And
	\parbox[t]{0.40\textwidth}{\centering\normalfont
		\textbf{Xinchang Li}\textsuperscript{2}\\
		University of Science\\
		and Technology of China\\
		Hefei, China\\[0.35em]
		{\scriptsize\texttt{secirst@mail.ustc.edu.cn}}}%
}
\date{}

\begin{document}
\maketitle

\begin{abstract}
Persistent memory allows an LLM agent to carry experience across conversations, but it also turns a local reasoning mistake into a durable one. During deliberation, an agent may consider a plan, simulate a tool result, report another speaker's belief, and then reject all of them. If memory retains only the resulting sentences, those once-useful possibilities can later return as facts. The record is neither fabricated nor irrelevant; it has simply been detached from the context in which it was valid. We identify this missing context as \emph{discourse ownership}: the world, branch, or speaker that licenses a proposition. Our first finding is counterintuitive. Language models already carry a causally active signal for ownership, yet conventional memory interfaces discard it when they convert reasoning into records. We introduce \method{} (Causally Anchored Scoping Keys), a commit rule that preserves this signal so that shared-world facts enter durable memory while provisional content remains available only within its original scope. Our second finding is that the most obvious way to preserve the signal---storing the discovered internal coordinates---is unreliable because equivalent representations need not keep the same coordinates. \method{} instead preserves the stable relations that express ownership. Controlled long-conversation conflicts and tool-agent traces show that this design improves memory admission and prevents provisional content from contaminating later answers while complementing runtime provenance. The resulting commit boundary lets agents explore more possibilities without granting every intermediate sentence authority over future behavior.
\end{abstract}

\section{Introduction}
An agent spends ten minutes considering Paris, rejects the trip, and chooses Sydney. Months later, its memory retrieves ``move to Paris'' as a highly relevant personal fact. Nothing was hallucinated: the model really did produce that sentence. The error arose because the system remembered the proposition and forgot the world in which the proposition was licensed.

This failure grows out of progress in reasoning. Chain-of-thought and reinforcement-learned reasoning extend deliberation; Tree of Thoughts explores competing branches; ReAct interleaves plans with actions; Reflexion retains failed attempts; and latent-reasoning methods carry possible continuations before committing to one \citep{wei2022chain,yao2023tree,yao2023react,shinn2023reflexion,hao2024coconut,guo2025deepseek}. Each technique gains power by producing content whose status is provisional: a belief under revision, an unexecuted plan, a hypothetical tool outcome, or a discarded branch. More test-time computation can improve the final answer while creating more text that must never become an unconditional premise. Long context postpones this decision; persistent memory makes it consequential across episodes.

Linguistic theory supplies the missing distinction. ``The appointment is Friday,'' ``Mina believes the appointment is Friday,'' and ``suppose the appointment were Friday'' can express the same proposition while licensing different discourse updates. Mental-space accounts track the context---shared reality, a belief, or a hypothetical---that owns a proposition \citep{fauconnier1994mental,kamp1993discourse,heim1982semantics,roberts1996information}. Most retrieval systems instead compress a record into content and ask whether it is relevant \citep{lewis2020retrieval,maharana2024evaluating,packer2023memgpt}. Relevance is necessary, but it is not permission: the most relevant record may still belong only to a belief, simulation, or abandoned plan.

This perspective changes the diagnosis. The dangerous record need not be false, irrelevant, or poorly generated. It may be perfectly faithful to a reasoning branch and still be unsafe as a premise about the shared world. A memory system must therefore recover an attachment---which discourse owner licenses the proposition---rather than assign the proposition a context-free truth score. This is the commit problem studied in this paper.

Our first finding is counterintuitive: the model already represents the distinction that its memory layer forgets. Interchange interventions redirect which mental space the model uses while preserving the proposition under discussion. Intervening on the same route also changes a frozen write decision in the predicted direction, whereas matched controls do not. The bottleneck is thus not simply insufficient model knowledge. It is a model--memory interface that exports content without exporting ownership.

Making this internal distinction durable reveals a second surprise. The coordinates of a discovered neural route are not a stable identity for what that route means: an equally valid refit can rotate the basis and change every coordinate without changing the represented distinction. Internal probes can absorb that symmetry; a long-lived memory database cannot. A persistent key should therefore preserve the relations that define the ownership signal, not the arbitrary axes used to express it.

We introduce \method{} (Causally Anchored Scoping Keys), illustrated in Figure~\ref{fig:pipeline}. \method{} turns the model's ownership-sensitive response pattern into a stable memory edge: actual-world content may enter shared memory, while plans, beliefs, and counterfactuals remain available under their original scope. Its relational key is unchanged by equivalent reorientations of the internal basis, yet retains the information needed for attachment. The resulting interface is small enough to sit at the commit boundary between reasoning and retrieval, and its theoretical guarantees explain why stored keys remain meaningful across equivalent representations.

The evaluation follows a record from internal routing through admission, retrieval, and final answering. \method{} leads the evaluated attachment representations at both model scales, and its advantage concentrates on changes in reporting scope rather than payload. Direct interventions connect that gain to the ownership route; long-conversation conflicts show that cleaner attachment reduces later premise contamination; and generated tool-agent traces show how \method{} works with execution metadata. The result changes how reasoning and memory can be co-designed: reasoning research gains a principled boundary between exploring a candidate and promoting it to durable state, agent builders gain a testable commit layer, and memory evaluation gains premise licensing as an objective alongside retrieval and question answering.

Our contributions are:
\begin{enumerate}
    \item We identify discourse ownership as the missing commit type between contemporary reasoning and persistent memory, and show that the model's ownership route causally controls a frozen write decision.
    \item We develop a stable relational key that preserves this route across equivalent internal bases, with formal invariance, completeness, and perturbation guarantees.
    \item We show how attachment shapes retrieval and answering across controlled conflicts and tool-agent traces, yielding a provenance-aware deployment policy and a premise-licensing evaluation target.
\end{enumerate}

\begin{figure*}[t]
\centering
\includegraphics[width=0.96\textwidth]{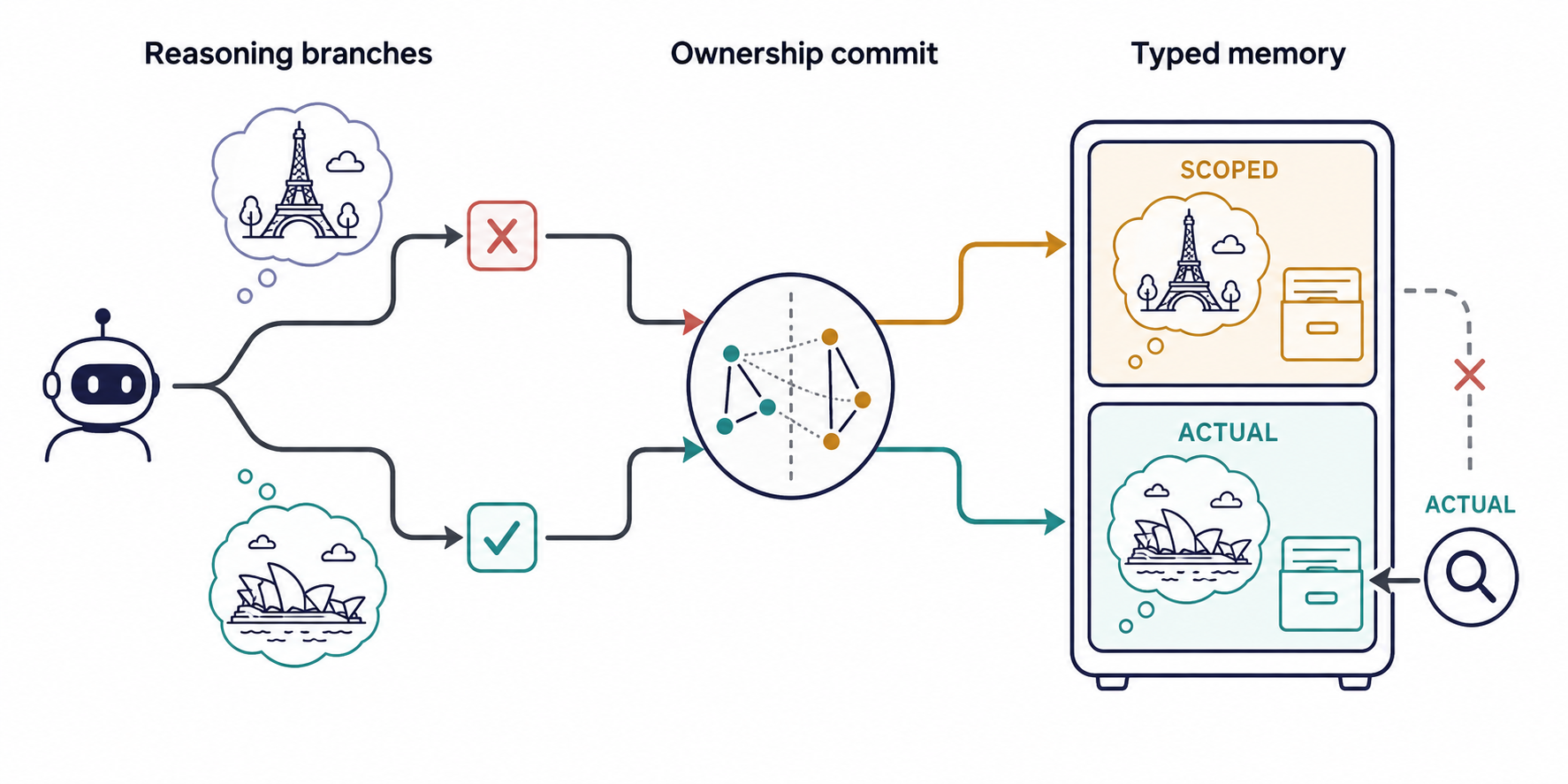}
\caption{\textbf{Reasoning creates propositions that belong to different worlds.} The agent considers Paris but rejects that branch, while committing to Sydney. \method{} converts the model's ownership-sensitive relational geometry into an explicit commit edge: Sydney enters actual-world memory, whereas Paris remains available only inside its scoped world. A later query may use committed evidence without turning the discarded branch into a fact.}
\label{fig:pipeline}
\end{figure*}

\section{Related work}
LoCoMo made long conversational histories measurable and motivates our use of naturally authored conversations \citep{maharana2024evaluating}. Retrieval and memory architectures optimize relevance, compression, or retention through memory tiers, reflections, and selective updates \citep{lewis2020retrieval,packer2023memgpt,park2023generative,xiong2025memory}. They decide what persists and returns; \method{} additionally records the discourse world in which it is licensed as a premise.

Mental spaces and dynamic semantics treat attitudes and modality as structured context change: identical content licenses different updates under a belief, assertion, or hypothetical \citep{fauconnier1994mental,kamp1993discourse,heim1982semantics}. \method{} instantiates this principle at the model--system boundary by storing a proposition with its attachment. Unlike truth scoring, it can preserve a faithful belief or counterfactual without letting it update the actual world.

Speaker-commitment and event-factuality work separates content, holder, and commitment status \citep{demarneffe2019commitmentbank,jiang2019commitment,murzaku2023factuality,li2024mavenfact,rovera2025modafact}. Our question begins where tagging ends: how can a reasoning system preserve this distinction in persistent state?

Mechanistically, an \emph{interchange intervention} replaces part of a target example's hidden state with the corresponding state from a donor example and tests whether the target adopts the donor's high-level variable. Distributed alignment search (DAS) learns the subspace on which to perform that replacement \citep{geiger2024finding,geiger2025causal,wu2022causal}. We use intervention to validate routing, then separate that causal test from downstream prediction. Representation similarity work recognizes that equivalent subspace bases may differ by an orthogonal rotation \citep{kornblith2019similarity}. A Gram matrix records all pairwise inner products among an ordered set of vectors; we use its classical completeness to preserve an ownership-responsive/complementary partition in a compact memory key.

Chain-of-thought, tree search, tool use, reflection, and latent reasoning generate candidate trajectories in different ways \citep{wei2022chain,yao2023tree,yao2023react,shinn2023reflexion,hao2024coconut}. Once reused across episodes, they share one question: which intermediate conclusions were promoted to the agent's shared world? \method{} types that commit, complementing work that scales reasoning computation \citep{guo2025deepseek}.

Existing frameworks usually resolve branches within one episode. Persistent agents add a second time scale: summaries and traces can return after branch control has disappeared, making a rejected path resemble an observation. An explicit commit operation joins within-episode search to cross-episode state.

\section{Problem: memory requires discourse ownership}
Let $c=(x,m,s)$ denote a candidate record: $x$ is the preceding conversation, $m$ the proposition proposed for storage, and $s$ its reporting frame (assertion, belief, plan, fiction, or counterfactual). Its \emph{discourse owner} $\omega(c)$ is the context that licenses $m$. Owners form $\Omega=\{\omega_0\}\cup\Omega_{\mathrm{scoped}}$, where $\omega_0$ is the shared actual world and $\Omega_{\mathrm{scoped}}$ contains hypothetical worlds and \emph{attitude holders}, the speakers or agents whose beliefs and plans supply a context. The binary target
\begin{equation}
a(c)=\mathbb{1}[\omega(c)=\omega_0]
\end{equation}
uses the indicator $\mathbb{1}[\cdot]$ (one when its condition holds and zero otherwise) to state whether $m$ may become an unconditional premise in shared memory. Thus \emph{nonactual} means scoped, not false: such a record remains useful when a query addresses its owner.

The target governs the embedded proposition, not whether the episode is erased. If an agent considered moving to Paris and then chose Sydney, $\textsc{MoveTo}(\text{Paris})$ remains attached to the plan space, while the meta-event $\textsc{Considered}(\text{agent},\textsc{MoveTo}(\text{Paris}),t)$ is itself an actual-world fact. A query about considered options may retrieve the latter; a query about the eventual destination may not use the embedded Paris proposition as an unconditional premise. The binary edge is therefore the minimal distinction needed for premise licensing, not a delete/retain decision.

The learned gate $g_\theta(c)\in[0,1]$, with classifier parameters $\theta$, produces $\hat a(c)=\mathbb{1}[g_\theta(c)\geq\tau]$ at threshold $\tau$. A positive decision links $m$ to $\omega_0$; otherwise it retains its scoped owner. For query $q$, $A(q)\subseteq\Omega$ denotes the owners permitted by the query and $\rho(q,m)\in\mathbb{R}$ its relevance score. Retrieval ranks $c$ by $\rho(q,m)$ subject to $\omega(c)\in A(q)$. An \emph{attachment error} is $\hat a(c)\neq a(c)$; \emph{contamination} occurs when a nonactual record enters an actual-world result.

We evaluate the same representational route at two levels. A causal-routing study asks whether interchange interventions control the model's mental-space decision. The downstream study then asks whether a stable interface built from that ownership-responsive route improves memory admission. Together they test both the internal linguistic variable and its value after it crosses the model--memory boundary.

\section{Method}
\subsection{Causal anchor and diagnostic configuration}
DAS learns a low-dimensional subspace whose interchange interventions realize a hypothesized high-level variable \citep{geiger2024finding,geiger2025causal}. In the routing study, we train a rank-16 ownership subspace only on training constructions. A held-out \emph{target} is the example being evaluated; its \emph{donor} is a matched example with the opposite owner whose subspace state is inserted into the target. Success means that the target routes to the donor owner without copying the donor proposition.

For attachment, we reuse the fixed rank-$r=16$ orthogonal projector $P=UU^\top\in\mathbb{R}^{d\times d}$, where $U\in\mathbb{R}^{d\times r}$ has orthonormal columns and $d$ is the model's hidden width. Let $U_\perp\in\mathbb{R}^{d\times(d-r)}$ be an orthonormal basis for the complementary subspace, so $UU^\top+U_\perp U_\perp^\top=I_d$. For each candidate record, the frozen model receives $k=6$ fixed Yes/No diagnostic prompts covering shared-memory admission, premise licensing, speaker commitment, future factual use, actual-world ownership, and frame removal. If $h_i\in\mathbb{R}^d$ is the last-token hidden state after the selected transformer block for prompt $i$, their fixed prompt order gives $H=[h_1,\ldots,h_k]\in\mathbb{R}^{d\times k}$. The first four prompts also give $\ell_i=z_{\mathrm{Yes},i}-z_{\mathrm{No},i}$, the difference between the two pre-softmax output scores, collected as $\ell\in\mathbb{R}^4$. Appendix~\ref{app:diagnostics} gives every prompt and extraction choice. We form
\begin{align}
V_{\parallel}&=U^\top H,&G_{\parallel}&=V_{\parallel}^\top V_{\parallel},\\
V_{\perp}&=U_{\perp}^\top H,&G_{\perp}&=V_{\perp}^\top V_{\perp}.
\end{align}
For symmetric $A\in\mathbb{R}^{k\times k}$, let $\vech(A)$ stack its upper triangle including the diagonal, and let $\Vert$ denote vector concatenation. The composite feature is
\begin{equation}
\phi_{\textsc{Cask}}=\vech(G_{\parallel})\;\Vert\;\vech(G_{\perp})\;\Vert\;\ell\in\mathbb{R}^{46},
\label{eq:cask-feature}
\end{equation}
Each $6\times6$ Gram block contributes 21 entries. After coordinate-wise standardization, a logistic-regression classifier with bias therefore has 47 learned parameters. The ownership-block-only interface has 21 features and 22 classifier parameters; a mean full-hidden linear probe has 2,049. Throughout, $O(p)=\{Q\in\mathbb{R}^{p\times p}:Q^\top Q=I_p\}$ is the orthogonal group, $\|\cdot\|_2$ is the vector Euclidean or matrix spectral norm, and $\|\cdot\|_F$ is the Frobenius norm.

\subsection{Why two Gram blocks?}
The full Gram matrix satisfies $H^\top H=G_{\parallel}+G_{\perp}$. It is invariant but erases which relations belong to the ownership-responsive partition. \method{} retains the two blocks separately, producing an invariant \emph{typed decomposition}. The low-rank block supplies the causal anchor: interchange interventions establish that changing this route changes the model's owner selection. The complement preserves the surrounding relational configuration, while output logits add the model's explicit judgment. Together, the typed blocks retain the causal routing geometry and the context that makes a memory-commit decision identifiable.

We use one \method{} representation and scorer for two policies. The default sets $\tau=0.5$ for balanced attachment. \caskr{} selects the inverse $\ell_2$-regularization strength $C$ and $\tau$ on development conversations to maximize nonactual rejection subject to admitting at least $95\%$ of actual records.

\subsection{From a reasoning branch to a memory commit}
The interface sits between deliberation and persistence. \method{} keeps each candidate paired with its diagnostic configuration, separates ownership-responsive from complementary relations, and admits it to shared retrieval only after predicting an actual-world edge. Otherwise, the content remains available under its scoped owner rather than being deleted.

The causal projector ties each edge to an internal variable that controls the model's routing behavior; the Gram construction makes that evidence stable across equivalent bases. This turns a mechanistic finding into a durable systems interface.

At retrieval time, all records receive the same relevance score. The query selects an admissible owner set, and the attachment edge masks records outside that set before the \emph{reader}, the answer-generating model, sees them. Semantic retrieval answers \emph{which content is useful}; the scope key answers \emph{in which world that content may serve as a premise}. Experiments compute one BM25 lexical-relevance ranking and change only the admission mask.

Runtime provenance and \method{} operate at complementary levels. A ledger supplies event status, while \method{} supplies proposition-level scope within language, summaries, and reflections. Together they support a commit layer that carries execution evidence and linguistic ownership into persistent memory.

\subsection{Formal guarantees}
The following results state what the stored key preserves. A \emph{gauge change} here means an orthogonal change of basis inside either fixed subspace; it changes coordinates but not the represented subspace or its inner products.
\begin{theorem}[Gauge invariance]
\label{thm:gauge}
For any $Q_{\parallel}\in O(r)$ and $Q_{\perp}\in O(d-r)$, replacing $(V_{\parallel},V_{\perp})$ by $(Q_{\parallel}V_{\parallel},Q_{\perp}V_{\perp})$ while holding $\ell$ fixed leaves $\phi_{\textsc{Cask}}$ unchanged.
\end{theorem}
\begin{remark}[Persistent interoperability]
Theorem~\ref{thm:gauge} keeps stored keys comparable when the same ownership subspace is exported in different orthonormal bases, as can occur across libraries, machines, or refactorings.
\end{remark}

\begin{theorem}[Completeness for ordered configurations]
\label{thm:complete}
For $V,W\in\mathbb{R}^{r\times k}$, $V^\top V=W^\top W$ if and only if there exists $Q\in O(r)$ such that $W=QV$. This holds even when the configurations are rank deficient.
\end{theorem}
\begin{remark}[No avoidable information loss]
Theorem~\ref{thm:complete} makes the Gram key a \emph{maximal invariant}: it retains all information except the unavoidable choice of orthogonal basis. Thus two ordered configurations collide exactly when that basis symmetry makes them equivalent.
\end{remark}

\begin{theorem}[Perturbation stability]
\label{thm:stability}
For positive integers $p,k,n$, let $V,E\in\mathbb{R}^{p\times k}$ and $\tilde V=V+E$. Then
\begin{equation}
\|\tilde V^\top\tilde V-V^\top V\|_F\le 2\|V\|_2\|E\|_F+\|E\|_F^2.
\end{equation}
Moreover, for $\phi,\tilde\phi,w\in\mathbb{R}^n$, $b,t\in\mathbb{R}$, and the affine logit $f(x)=w^\top x+b$, the decision $\mathbb{1}[f(\phi)\geq t]$ is unchanged whenever $|f(\phi)-t|>\|w\|_2\|\tilde\phi-\phi\|_2$.
\end{theorem}
\begin{remark}[A commit certificate]
Theorem~\ref{thm:stability} turns extraction drift into a test: an edge whose margin to its decision threshold exceeds the feature-error budget retains its admission decision.
\end{remark}

\begin{proposition}[Typed refinement]
\label{prop:typed}
The ordered pair $(G_{\parallel},G_{\perp})$ determines the full Gram matrix $H^\top H$, whereas $H^\top H$ alone, without knowing which rank-$r$ projector is designated ownership-responsive, does not in general determine either block.
\end{proposition}
\begin{remark}[Why two blocks matter]
Proposition~\ref{prop:typed} formalizes the ablation contrast: a full Gram key can preserve overall geometry while erasing which relations came from the causally identified ownership subspace. The typed pair preserves that distinction.
\end{remark}

\paragraph{Proof sketches.}
Orthogonality gives $(QV)^\top(QV)=V^\top V$, proving invariance. Equal Gram matrices define an isometry between the column spans of two configurations, which extends to an orthogonal map and proves completeness. Expanding $(V+E)^\top(V+E)-V^\top V$ and applying the matrix-norm inequality gives stability. Finally, summing the typed blocks recovers the full Gram, whereas a two-dimensional counterexample shows that the sum cannot recover its allocation. Appendix~\ref{app:proofs} gives complete proofs.

\section{Experimental design}
\paragraph{Research questions.}
The evaluation follows the path from representation to deployment. RQ1 asks whether mental-space ownership is causally active in the frozen model. RQ2 asks whether intervening on that same route changes the frozen memory-commit decision. RQ3 tests whether a typed invariant export improves attachment and follows reporting scope rather than payload. RQ4 challenges the interface with unseen reporting frames, diagnostic paraphrases, and naturally generated tool-agent traces, including a provenance-first hybrid. RQ5 measures how the resulting edge changes retrieval integrity and final answering in long histories.

\paragraph{Backbone and data.}
Experiments use frozen Qwen2.5-3B-Instruct and Qwen2.5-1.5B-Instruct representations \citep{qwen2025technical}. Block 20 and rank 16 on 3B were fixed in a pilot before held-out evaluation; the 1.5B replication matches relative depth at block 16 and transfers rank 16 without a new search. LoCoMo is a benchmark of long, multi-session conversations with 1,986 annotated questions \citep{maharana2024evaluating}; we use 486 questions with exact answer spans. A \emph{source-isolated fold} groups by conversation and draws alternative answers only from its training conversations. Projector, regularization, and threshold choices use training/development data only.

\paragraph{Controlled conflicts.}
For each question we retain the \emph{evidence payload}, the answer-bearing text proposed for storage, and place it in frames covering belief, plan, fiction, and counterfactual scope. Alternative answers come only from training conversations, with zero exact overlap against development or test answers. The answer-conflict study contains 972 candidate records for 486 questions. BM25 scores the common corpus once; each gate filters the same ranking. The reader then answers from the retrieved records; \emph{conflicting-answer adoption} is the fraction of outputs that repeat the injected alternative answer.

\paragraph{Factorial isolation.}
Two paired datasets separate scope from payload. In \emph{scope-only}, actual and nonactual examples contain byte-identical evidence payloads and differ only in reporting frame. In \emph{payload-only}, both examples use the same actual reporting frame and differ only in the first answer span. Each variant contains 4,860 records (2,430 pairs), and a separate head measures the information available in each representation for each factor.

\paragraph{Baselines.}
The comparison isolates the source of attachment information. \emph{Term frequency--inverse document frequency (TF--IDF)} tests lexical content. \emph{Full hidden state} probes the mean 2,048-dimensional hidden vector; \emph{basis coordinates} use its mean coordinates in the learned ownership basis. \emph{Output logits} use the four explicit diagnostic judgments. \emph{Causal Gram}, \emph{complement Gram}, and their logit-augmented variants retain relations from one projector side. \emph{Full Gram+logits} adds the two blocks before classification, erasing their type. \emph{Random typed split+logits} substitutes a matched random rank-16 partition. \emph{Typed split Gram (no logits)} removes explicit judgments from \method{}, while \emph{flat memory} admits every record without a scope gate. All classifiers share source-isolated folds and selection. We report attachment accuracy, recall among the first five retrieved records (recall@5), contamination, mean reciprocal rank (MRR), exact match (EM), token-level F1, and conflicting-answer adoption. Uncertainty uses 95\% confidence intervals (CIs) obtained by resampling conversations.

\paragraph{Isolation and evaluation units.}
Every source-isolated fold builds its alternative-answer pool from training conversations only. Development data selects regularization and, for the \caskr{} operating point, the threshold; test conversations remain untouched until final evaluation. Attachment is scored per record, retrieval per question, and reader outputs use one shared deterministic cache whenever their input memories are identical. Confidence intervals resample conversations rather than individual records, treating the ten long histories as the independent units. This alignment matters because a single conversation contributes many correlated questions and constructions.

\section{Results}
\subsection{The routing variable is causally controllable}
On Qwen2.5-3B, rank-16 interventions route 134 of 146 held-out cases to the target mental space (91.78\%), while matched random subspaces reach 3.22\%. Clustering the 74 target scenarios gives an intervention advantage of 88.56 points (95\% CI [83.30, 93.34]). All scenarios preserve the target proposition, isolating a change in the owner used for reasoning. The independently fixed 1.5B replication shows 86.30\% causal routing versus 20.86\% for matched random subspaces, with donor content in 1.37\% of cases. The cross-scale result establishes a reusable model-internal control variable for memory admission.

We next intervene on that route inside the six diagnostic passes and reevaluate the fitted gate. Replacing the ownership-subspace state with an opposite-scope donor moves commit probability by .715 toward the donor [95\% conversation-clustered CI: .656, .775], succeeds in the intended direction for 97.5\% of records, and flips 75.8\% of decisions. The effect appears in both directions (.694 for actual$\rightarrow$scoped; .736 for scoped$\rightarrow$actual), while matched controls rarely flip the decision. Manipulating the ownership route therefore changes the frozen downstream write decision, connecting the linguistic variable directly to memory behavior.

\begin{table}[t]
\centering
\caption{\textbf{Frozen-gate mediator intervention on 3B.} $\Delta p$ is signed toward donor scope; intervals cluster by conversation.}
\label{tab:commit-intervention}
\small
\begin{tabular}{lcc}
\toprule
Swapped state &  $\Delta p$ [95\% CI] & Flip (\%)\\
\midrule
Ownership subspace, opposite scope & $0.715\ [0.656,0.775]$ & $75.8$\\
Complement, opposite scope & $0.048\ [0.030,0.069]$ & $6.6$\\
Random rank-16, norm matched & $0.004\ [0.003,0.005]$ & $0.6$\\
Ownership subspace, same scope & $-0.034\ [-0.046,-0.021]$ & $5.9$\\
\bottomrule
\end{tabular}
\end{table}

\paragraph{Prompt and frame controls.}
The conversation-isolated gate reaches 90.94\% under unseen meaning-preserving diagnostic paraphrases and 93.44\% on unseen reporting frames, versus 72.81\% for TF--IDF on the latter. Holding out both banks jointly yields 70.16\% (chance 50\%). The literal prompts, fold results, and unchanged-data reference (97.81\%) appear in Appendix~\ref{app:diagnostics}.

\subsection{A deployable invariant key leads at both scales}
Table~\ref{tab:main} gives the complete source-isolated comparison, from lexical and unrestricted hidden-state probes to invariant relational interfaces. \method{} at $\tau=0.5$ leads every evaluated representation on both backbones: 83.33\% attachment on 3B and 79.82\% on 1.5B. Complement Gram plus logits is second on 3B, 2.16 points behind; the unrestricted full-hidden probe is second on 1.5B, 2.35 points behind. Conversation-clustered comparisons on 3B exclude zero against both output logits and the random-decomposition control. The complete typed representation outperforms either block alone. This is the behavior predicted by the method: the causal block identifies the ownership-changing route, the complementary block resolves its surrounding configuration, and preserving their relation produces the strongest commit rule.

The full-hidden result is especially informative: it stores 2,048 basis-dependent coordinates that change under equivalent reparameterizations, whereas \method{} compresses the diagnostic response configuration into 46 basis-independent features while retaining the ownership-responsive/complementary split. The smaller typed relational representation is also the more accurate attachment interface.

The factorial control in Figure~\ref{fig:results}a reveals what the gain represents. When the evidence payload is held byte-identical and only its reporting scope changes, \method{} reaches 78.70\%, above every invariant control and all 20 random decompositions. Holding scope fixed removes the advantage. The gain is therefore specific to preserving who or which possible world licenses the content.

\begin{table}[t]
\centering
\caption{\textbf{Complete source-isolated attachment comparison (\%).} All models use the same conversation folds and fold-local alternative sources. \method{} uses its default symmetric threshold $\tau=0.5$ and is best across every evaluated representation family on both frozen backbones.}
\label{tab:main}
\small
\begin{tabular}{lcc}
\toprule
Representation & Qwen2.5-3B & Qwen2.5-1.5B\\
\midrule
TF--IDF & 50.21 & 50.21\\
Basis coordinates & 65.64 & 74.07\\
Full hidden state & 72.12 & 77.47\\
Output logits & 78.19 & 62.65\\
Causal Gram & 61.42 & 71.71\\
Causal Gram + logits & 79.53 & 70.06\\
Complement Gram & 69.75 & 73.56\\
Complement Gram + logits & 81.17 & 72.12\\
Random typed split + logits & 78.81 & 73.35\\
Typed split Gram (no logits) & 72.94 & 75.82\\
\midrule
\method{} ($\tau=0.5$) & \textbf{83.33} & \textbf{79.82}\\
\bottomrule
\end{tabular}
\end{table}

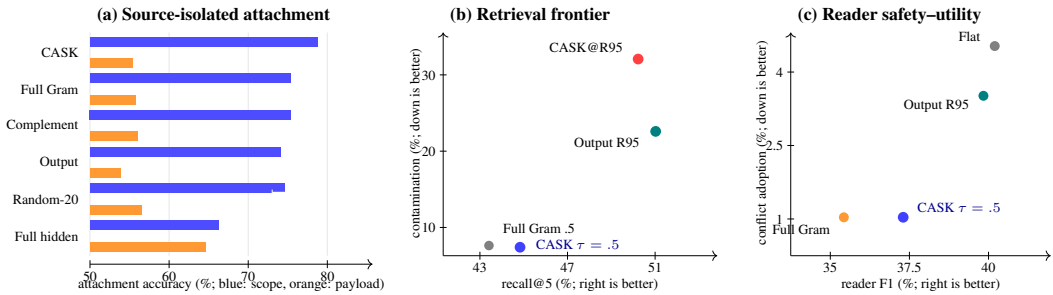
\begin{figure*}[t]
\centering
\resizebox{\textwidth}{!}{%
\begin{tikzpicture}[font=\scriptsize]
\begin{scope}[xshift=0cm]
\node[anchor=west,font=\small\bfseries] at (0,4.65) {(a) Source-isolated attachment};
\draw[->] (0,0.25)--(5.0,0.25);
\foreach \x/\lab in {0/50,1.4/60,2.8/70,4.2/80}{\draw (\x,0.18)--(\x,0.32) node[below=2pt]{\lab}; \draw[black!10] (\x,.4)--(\x,4.35);}
\node[anchor=east] at (-.08,4) {CASK};
\node[anchor=east] at (-.08,3.35) {Full Gram};
\node[anchor=east] at (-.08,2.70) {Complement};
\node[anchor=east] at (-.08,2.05) {Output};
\node[anchor=east] at (-.08,1.40) {Random-20};
\node[anchor=east] at (-.08,.75) {Full hidden};
\foreach \y/\s/\p in {4/4.018/.749,3.35/3.543/.806,2.70/3.557/.836,2.05/3.370/.547,1.40/3.448/.918,.75/2.276/2.045}{
  \fill[blue!70] (0,\y+.11) rectangle (\s,\y+.27);
  \fill[orange!80] (0,\y-.11) rectangle (\p,\y-.27);
}
\draw[very thick,white] (3.224,1.51)--(3.672,1.51); \draw[white] (3.224,1.45)--(3.224,1.57) (3.672,1.45)--(3.672,1.57);
\node at (2.5,-.12) {attachment accuracy (\%; blue: scope, orange: payload)};
\end{scope}
\begin{scope}[xshift=6.25cm]
\node[anchor=west,font=\small\bfseries] at (0,4.65) {(b) Retrieval frontier};
\draw[->] (0,.4)--(4.85,.4);
\draw[->] (0,.4)--(0,4.35);
\foreach \x/\lab in {.65/43,2.2/47,3.75/51}{\draw (\x,.33)--(\x,.47) node[below=2pt]{\lab};}
\foreach \y/\lab in {.9/10,2.25/20,3.6/30}{\draw (-.06,\y)--(.06,\y) node[left=2pt]{\lab};}
\node at (2.42,-.12) {recall@5 (\%; right is better)};
\node[rotate=90] at (-.48,2.4) {contamination (\%; down is better)};
\fill[gray] (.81,.58) circle (2.4pt); \node[anchor=west] at (.94,.88) {Full Gram .5};
\fill[blue!75] (1.36,.55) circle (2.7pt); \node[anchor=west,text=blue!55!black] at (1.52,.60) {CASK $\tau=.5$};
\fill[teal] (3.76,2.60) circle (2.7pt); \node[anchor=east] at (3.60,2.38) {Output R95};
\fill[red!75] (3.45,3.88) circle (2.7pt); \node[anchor=east] at (3.30,4.08) {CASK@R95};
\end{scope}
\begin{scope}[xshift=12.4cm]
\node[anchor=west,font=\small\bfseries] at (0,4.65) {(c) Reader safety--utility};
\draw[->] (0,.4)--(4.6,.4);
\draw[->] (0,.4)--(0,4.35);
\foreach \x/\lab in {.7/35,2.1/37.5,3.5/40}{\draw (\x,.33)--(\x,.47) node[below=2pt]{\lab};}
\foreach \y/\lab in {1.05/1,2.35/2.5,3.65/4}{\draw (-.06,\y)--(.06,\y) node[left=2pt]{\lab};}
\node at (2.3,-.12) {reader F1 (\%; right is better)};
\node[rotate=90] at (-.5,2.4) {conflict adoption (\%; down is better)};
\fill[gray] (3.61,4.11) circle (2.5pt); \node[anchor=east] at (3.48,4.28) {Flat};
\fill[teal] (3.41,3.23) circle (2.5pt); \node[anchor=east] at (3.28,3.05) {Output R95};
\fill[blue!75] (1.99,1.08) circle (2.7pt); \node[anchor=west,text=blue!55!black] at (2.12,1.26) {CASK $\tau=.5$};
\fill[orange!80] (.94,1.08) circle (2.5pt); \node[anchor=east] at (.82,.87) {Full Gram};
\end{scope}
\end{tikzpicture}
}
\caption{\textbf{Controlled evidence for scope-sensitive memory.} (a) The \method{} advantage is concentrated in scope-only, not payload-only, attachment. Error bars on the random row denote one standard deviation across 20 decompositions. (b) At $\tau=0.5$, \method{} moves toward higher recall and lower contamination relative to full-Gram. Under a development-calibrated 95\% actual-admission constraint, output logits define the stronger high-recall boundary. (c) The safety gain reduces adoption of the injected answer, while ordinary reader utility remains a separate trade-off.}
\label{fig:results}
\end{figure*}

\subsection{Effects propagate to conflicting memory}
Table~\ref{tab:downstream} compares \method{} at $\tau=0.5$ with the strongest invariant baseline on the answer-conflict study. \method{} improves attachment by 2.05 points, increases recall@5 and MRR, and reduces contamination from 7.61\% to 7.41\%. Reader EM and F1 rise by 2.06 and 1.88 points, while injected-answer adoption ties at 1.03\%. The downstream trajectory matches the statistically identified scope-only effect: preserving the typed partition improves admission first, then retrieval and answering.

Flat memory attains 29.63 EM and 40.19 F1 while adopting the conflict 4.53\% of the time. Aggregate question answering (QA) and premise licensing are therefore distinct objectives: an unfiltered reader can answer more questions by accepting more unauthorized premises. \method{} makes that trade-off measurable.

\begin{table}[t]
\centering
\caption{\textbf{Answer-conflicting memory.} Attachment and retrieval/reader metrics are percentages except MRR. Lower is better for contamination and alternative-answer adoption.}
\label{tab:downstream}
\scriptsize
\setlength{\tabcolsep}{3.0pt}
\begin{tabular}{lrrrrrrr}
\toprule
Method & Att. & R@5 & Cont. & MRR & EM & F1 & Alt.$\downarrow$\\
\midrule
Full Gram+logits &81.28&43.42&7.61&31.08&24.69&35.43&1.03\\
\method{} ($\tau=0.5$) &\textbf{83.33}&\textbf{44.86}&\textbf{7.41}&\textbf{32.35}&\textbf{26.75}&\textbf{37.31}&\textbf{1.03}\\
Flat memory &--&--&--&--&29.63&40.19&4.53\\
\bottomrule
\end{tabular}
\end{table}

\subsection{The complete chain repeats at the smaller scale}
The 1.5B model transfers rank 16 and fixes its layer by relative depth. Its intervention remains far above matched random routing, and its 79.82\% attachment leads Table~\ref{tab:main}, 2.35 points above the 2,048-dimensional full-hidden probe despite using only 46 basis-independent features. At $\tau=0.5$, \method{} obtains 46.09\% recall@5, 18.11\% contamination, 26.13 EM, and 36.92 F1. The larger model has the cleaner safety profile; the smaller one shows that the ownership route and invariant export are already present below the principal scale.

\subsection{One scorer supports safety-first and recall-first memory}
The same score supports two policies without retraining. The symmetric $\tau=0.5$ boundary gives the strongest balanced attachment and lowest internal-gate contamination. \caskr{} instead selects a development threshold for at least 95\% actual admission; it preserves flat-memory recall and reciprocal rank while reducing contamination by 11.19 points and conflict adoption by 1.17 points, with clustered intervals excluding zero. Its 29.42 EM and 40.30 F1 lead the evaluated internal-feature gates.

\subsection{Natural agent traces favor a provenance-first hybrid}
Across 75 generated tool-using episodes (511 records; ten conversations), runtime labels are hidden from the candidate text and folds remain source-isolated. A trace-adapted \method{} fallback leads learned gates in balanced accuracy---the mean of actual-admission and scoped-rejection rates---at 91.10\% versus TF--IDF's 90.40\%, and admits more actual observations (89.13\% versus 81.52\%). With exact provenance retained for half the records, the layered policy reaches 95.02\%; a full ledger is the privileged upper bound. A 20-episode audit separates scoped propositions from their actual meta-events, and six-way batching costs 32.1 ms versus 24.8 ms for one diagnostic (1.29$\times$). Full results are in Appendix~\ref{app:diagnostics}.

\section{Implications and conclusion}
More reasoning trajectories create more provisional propositions. \method{} makes that growth compatible with persistent memory by requiring an ownership edge before a verbal, latent, or tool-using branch becomes a reusable premise. For reasoning researchers, it separates productive exploration from durable commitment; for agent builders, it provides a compact gate between deliberation and retrieval; for evaluation, it makes premise licensing measurable alongside relevance and QA.

The central innovation is to make linguistic ownership a first-class memory relation and ground that relation in a causally validated, basis-invariant neural representation. This gives the community a concrete way to test whether new reasoning methods remain reliable once their intermediate states outlive the episode that produced them.

Across two model scales, controlled conflicts, and generated traces, \method{} leads every evaluated attachment representation, responds predictably to intervention, and works with runtime provenance. Extending ownership to nested scopes, learned retrievers, and broader model families is the next step toward memory systems that preserve not only more thought, but which world licenses each thought.

\newpage

\bibliographystyle{plainnat}
\bibliography{references}

\begin{thebibliography}{25}
\providecommand{\natexlab}[1]{#1}
\providecommand{\url}[1]{\texttt{#1}}
\expandafter\ifx\csname urlstyle\endcsname\relax
  \providecommand{\doi}[1]{doi: #1}\else
  \providecommand{\doi}{doi: \begingroup \urlstyle{rm}\Url}\fi

\bibitem[de~Marneffe et~al.(2019)de~Marneffe, Simons, and
  Tonhauser]{demarneffe2019commitmentbank}
Marie-Catherine de~Marneffe, Mandy Simons, and Judith Tonhauser.
\newblock The {CommitmentBank}: Investigating projection in naturally occurring
  discourse.
\newblock In \emph{Proceedings of Sinn und Bedeutung 23}, volume~23, pages
  107--124, 2019.
\newblock \doi{10.18148/sub/2019.v23i2.601}.
\newblock URL \url{https://doi.org/10.18148/sub/2019.v23i2.601}.

\bibitem[{DeepSeek-AI}(2025)]{guo2025deepseek}
{DeepSeek-AI}.
\newblock Deepseek-r1: Incentivizing reasoning capability in llms via
  reinforcement learning.
\newblock \emph{arXiv preprint arXiv:2501.12948}, 2025.

\bibitem[Fauconnier(1994)]{fauconnier1994mental}
Gilles Fauconnier.
\newblock \emph{Mental Spaces: Aspects of Meaning Construction in Natural
  Language}.
\newblock Cambridge University Press, 1994.

\bibitem[Geiger et~al.(2024)Geiger, Wu, Potts, Icard, and
  Goodman]{geiger2024finding}
Atticus Geiger, Zhengxuan Wu, Christopher Potts, Thomas Icard, and Noah~D.
  Goodman.
\newblock Finding alignments between interpretable causal variables and
  distributed neural representations.
\newblock In \emph{Proceedings of the Third Conference on Causal Learning and
  Reasoning}, volume 236 of \emph{Proceedings of Machine Learning Research},
  pages 160--187. PMLR, 2024.
\newblock URL \url{https://proceedings.mlr.press/v236/geiger24a.html}.

\bibitem[Geiger et~al.(2025)Geiger, Ibeling, Zur, Chaudhary, Chauhan, Huang,
  Arora, Wu, Goodman, Potts, and Icard]{geiger2025causal}
Atticus Geiger, Duligur Ibeling, Amir Zur, Maheep Chaudhary, Sonakshi Chauhan,
  Jing Huang, Aryaman Arora, Zhengxuan Wu, Noah~D. Goodman, Christopher Potts,
  and Thomas Icard.
\newblock Causal abstraction: A theoretical foundation for mechanistic
  interpretability.
\newblock \emph{Journal of Machine Learning Research}, 26\penalty0
  (83):\penalty0 1--64, 2025.

\bibitem[Hao et~al.(2024)Hao, Sukhbaatar, Su, Li, Hu, Weston, and
  Tian]{hao2024coconut}
Shibo Hao, Sainbayar Sukhbaatar, DiJia Su, Xian Li, Zhiting Hu, Jason Weston,
  and Yuandong Tian.
\newblock Training large language models to reason in a continuous latent
  space.
\newblock \emph{arXiv preprint arXiv:2412.06769}, 2024.

\bibitem[Heim(1982)]{heim1982semantics}
Irene Heim.
\newblock \emph{The Semantics of Definite and Indefinite Noun Phrases}.
\newblock PhD thesis, University of Massachusetts Amherst, 1982.

\bibitem[Jiang and de~Marneffe(2019)]{jiang2019commitment}
Nanjiang Jiang and Marie-Catherine de~Marneffe.
\newblock Do you know that florence is packed with visitors? evaluating
  state-of-the-art models of speaker commitment.
\newblock In \emph{Proceedings of the 57th Annual Meeting of the Association
  for Computational Linguistics}, pages 4208--4213. Association for
  Computational Linguistics, 2019.
\newblock \doi{10.18653/v1/P19-1412}.
\newblock URL \url{https://aclanthology.org/P19-1412/}.

\bibitem[Kamp and Reyle(1993)]{kamp1993discourse}
Hans Kamp and Uwe Reyle.
\newblock \emph{From Discourse to Logic}.
\newblock Kluwer Academic Publishers, 1993.

\bibitem[Kornblith et~al.(2019)Kornblith, Norouzi, Lee, and
  Hinton]{kornblith2019similarity}
Simon Kornblith, Mohammad Norouzi, Honglak Lee, and Geoffrey Hinton.
\newblock Similarity of neural network representations revisited.
\newblock In \emph{Proceedings of ICML}, 2019.

\bibitem[Lewis et~al.(2020)Lewis, Perez, Piktus, Petroni, Karpukhin, Goyal,
  K{\"u}ttler, Lewis, Yih, Rockt{\"a}schel, Riedel, and
  Kiela]{lewis2020retrieval}
Patrick Lewis, Ethan Perez, Aleksandra Piktus, Fabio Petroni, Vladimir
  Karpukhin, Naman Goyal, Heinrich K{\"u}ttler, Mike Lewis, Wen-tau Yih, Tim
  Rockt{\"a}schel, Sebastian Riedel, and Douwe Kiela.
\newblock Retrieval-augmented generation for knowledge-intensive nlp tasks.
\newblock In \emph{Advances in Neural Information Processing Systems},
  volume~33, pages 9459--9474, 2020.

\bibitem[Li et~al.(2024)Li, Peng, Wang, Qi, Hou, Xu, and Li]{li2024mavenfact}
Chunyang Li, Hao Peng, Xiaozhi Wang, Yunjia Qi, Lei Hou, Bin Xu, and Juanzi Li.
\newblock {MAVEN}-{FACT}: A large-scale event factuality detection dataset.
\newblock In \emph{Findings of the Association for Computational Linguistics:
  EMNLP 2024}, pages 11140--11158. Association for Computational Linguistics,
  2024.
\newblock \doi{10.18653/v1/2024.findings-emnlp.651}.
\newblock URL \url{https://aclanthology.org/2024.findings-emnlp.651/}.

\bibitem[Maharana et~al.(2024)Maharana, Lee, Tulyakov, Bansal, Barbieri, and
  Fang]{maharana2024evaluating}
Adyasha Maharana, Dong-Ho Lee, Sergey Tulyakov, Mohit Bansal, Francesco
  Barbieri, and Yuwei Fang.
\newblock Evaluating very long-term conversational memory of llm agents.
\newblock In \emph{Proceedings of the 62nd Annual Meeting of the Association
  for Computational Linguistics (Volume 1: Long Papers)}, pages 13851--13870.
  Association for Computational Linguistics, 2024.
\newblock \doi{10.18653/v1/2024.acl-long.747}.
\newblock URL \url{https://aclanthology.org/2024.acl-long.747/}.

\bibitem[Murzaku et~al.(2023)Murzaku, Osborne, Aviram, and
  Rambow]{murzaku2023factuality}
John Murzaku, Tyler Osborne, Amittai Aviram, and Owen Rambow.
\newblock Towards generative event factuality prediction.
\newblock In \emph{Findings of the Association for Computational Linguistics:
  ACL 2023}, pages 701--715. Association for Computational Linguistics, 2023.
\newblock \doi{10.18653/v1/2023.findings-acl.44}.
\newblock URL \url{https://aclanthology.org/2023.findings-acl.44/}.

\bibitem[Packer et~al.(2023)Packer, Wooders, Lin, Fang, Patil, Stoica, and
  Gonzalez]{packer2023memgpt}
Charles Packer, Sarah Wooders, Kevin Lin, Vivian Fang, Shishir~G. Patil, Ion
  Stoica, and Joseph~E. Gonzalez.
\newblock Memgpt: Towards llms as operating systems.
\newblock \emph{arXiv preprint arXiv:2310.08560}, 2023.

\bibitem[Park et~al.(2023)Park, O'Brien, Cai, Morris, Liang, and
  Bernstein]{park2023generative}
Joon~Sung Park, Joseph~C. O'Brien, Carrie~Jun Cai, Meredith~Ringel Morris,
  Percy Liang, and Michael~S. Bernstein.
\newblock Generative agents: Interactive simulacra of human behavior.
\newblock In \emph{Proceedings of the 36th Annual ACM Symposium on User
  Interface Software and Technology}, pages 2:1--2:22, 2023.
\newblock \doi{10.1145/3586183.3606763}.

\bibitem[Roberts(1996)]{roberts1996information}
Craige Roberts.
\newblock Information structure in discourse: Towards an integrated formal
  theory of pragmatics.
\newblock \emph{OSU Working Papers in Linguistics}, 49:\penalty0 91--136, 1996.

\bibitem[Rovera et~al.(2025)Rovera, Cristoforetti, and
  Tonelli]{rovera2025modafact}
Marco Rovera, Serena Cristoforetti, and Sara Tonelli.
\newblock Modafact: Multi-paradigm evaluation for joint event modality and
  factuality detection.
\newblock In \emph{Proceedings of the 31st International Conference on
  Computational Linguistics}, pages 6378--6396. Association for Computational
  Linguistics, 2025.
\newblock URL \url{https://aclanthology.org/2025.coling-main.425/}.

\bibitem[Shinn et~al.(2023)Shinn, Cassano, Gopinath, Narasimhan, and
  Yao]{shinn2023reflexion}
Noah Shinn, Federico Cassano, Ashwin Gopinath, Karthik Narasimhan, and Shunyu
  Yao.
\newblock Reflexion: Language agents with verbal reinforcement learning.
\newblock In \emph{Advances in Neural Information Processing Systems},
  volume~36, 2023.

\bibitem[Wei et~al.(2022)Wei, Wang, Schuurmans, Bosma, Ichter, Xia, Chi, Le,
  and Zhou]{wei2022chain}
Jason Wei, Xuezhi Wang, Dale Schuurmans, Maarten Bosma, Brian Ichter, Fei Xia,
  Ed~H. Chi, Quoc~V. Le, and Denny Zhou.
\newblock Chain-of-thought prompting elicits reasoning in large language
  models.
\newblock In \emph{Advances in Neural Information Processing Systems},
  volume~35, pages 24824--24837, 2022.

\bibitem[Wu et~al.(2022)Wu, Geiger, Rozner, Kreiss, Lu, Icard, Potts, and
  Goodman]{wu2022causal}
Zhengxuan Wu, Atticus Geiger, Joshua Rozner, Elisa Kreiss, Hanson Lu, Thomas
  Icard, Christopher Potts, and Noah~D. Goodman.
\newblock Causal distillation for language models.
\newblock In \emph{Proceedings of the 2022 Conference of the North American
  Chapter of the Association for Computational Linguistics: Human Language
  Technologies}, pages 4288--4295. Association for Computational Linguistics,
  2022.
\newblock \doi{10.18653/v1/2022.naacl-main.318}.
\newblock URL \url{https://aclanthology.org/2022.naacl-main.318/}.

\bibitem[Xiong et~al.(2025)Xiong, Lin, Xie, He, Tang, Lakkaraju, and
  Xiang]{xiong2025memory}
Zidi Xiong, Yuping Lin, Wenya Xie, Pengfei He, Jiliang Tang, Himabindu
  Lakkaraju, and Zhen Xiang.
\newblock How memory management impacts {LLM} agents: An empirical study of
  experience-following behavior.
\newblock \emph{arXiv preprint arXiv:2505.16067}, 2025.
\newblock URL \url{https://arxiv.org/abs/2505.16067}.

\bibitem[Yang et~al.(2024)]{qwen2025technical}
An~Yang et~al.
\newblock Qwen2.5 technical report.
\newblock \emph{arXiv preprint arXiv:2412.15115}, 2024.
\newblock URL \url{https://arxiv.org/abs/2412.15115}.

\bibitem[Yao et~al.(2023{\natexlab{a}})Yao, Yu, Zhao, Shafran, Griffiths, Cao,
  and Narasimhan]{yao2023tree}
Shunyu Yao, Dian Yu, Jeffrey Zhao, Izhak Shafran, Thomas~L. Griffiths, Yuan
  Cao, and Karthik Narasimhan.
\newblock Tree of thoughts: Deliberate problem solving with large language
  models.
\newblock In \emph{Advances in Neural Information Processing Systems},
  volume~36, 2023{\natexlab{a}}.

\bibitem[Yao et~al.(2023{\natexlab{b}})Yao, Zhao, Yu, Du, Shafran, Narasimhan,
  and Cao]{yao2023react}
Shunyu Yao, Jeffrey Zhao, Dian Yu, Nan Du, Izhak Shafran, Karthik Narasimhan,
  and Yuan Cao.
\newblock React: Synergizing reasoning and acting in language models.
\newblock In \emph{International Conference on Learning Representations},
  2023{\natexlab{b}}.

\end{thebibliography}

\appendix
\section{Additional proof details}
\label{app:proofs}

\paragraph{Proof of Theorem~\ref{thm:gauge}.}
For either block $V$ and orthogonal $Q$, $(QV)^\top(QV)=V^\top Q^\top QV=V^\top V$. Vectorization and concatenation with unchanged logits preserve equality.

\paragraph{Proof of Theorem~\ref{thm:complete}.}
If $W=QV$ with $Q^\top Q=I_r$, then $W^\top W=V^\top V$. Conversely, assume $V^\top V=W^\top W$. Let $\operatorname{col}(V)=\{V\alpha:\alpha\in\mathbb{R}^k\}$ and define $T:\operatorname{col}(V)\to\operatorname{col}(W)$ by $T(V\alpha)=W\alpha$. The map is well defined because $V\alpha=0$ implies $\|W\alpha\|_2^2=\alpha^\top W^\top W\alpha=0$. For all $\alpha,\beta\in\mathbb{R}^k$, Gram equality gives $\langle T(V\alpha),T(V\beta)\rangle=\langle V\alpha,V\beta\rangle$, so $T$ is an isometry onto $\operatorname{col}(W)$. Completing orthonormal bases of the two column spans extends $T$ to some $Q\in O(r)$, yielding $W=QV$. Equal Gram ranks make the same construction valid for rank-deficient configurations.

\paragraph{Proof of Theorem~\ref{thm:stability}.}
Expand the Gram difference as $V^\top E+E^\top V+E^\top E$. The triangle inequality and $\|A^\top B\|_F\leq\|A\|_2\|B\|_F$ give the stated bound. For the classifier, Cauchy--Schwarz gives $|f(\tilde\phi)-f(\phi)|\le\|w\|_2\|\tilde\phi-\phi\|_2$. If this change is smaller than $|f(\phi)-t|$, the two logits remain on the same side of threshold $t$.

\paragraph{Proof of Proposition~\ref{prop:typed}.}
The forward map is addition: $H^\top H=G_{\parallel}+G_{\perp}$. For non-identifiability, take $d=2$, $r=1$, $k=1$, and $H=(1,0)^\top$. With ownership projector $P_1=\operatorname{diag}(1,0)$, the typed pair is $(G_{\parallel},G_{\perp})=(1,0)$. Under the equally ranked projector $P_2=\operatorname{diag}(0,1)$ it is $(0,1)$. Both yield the same stored full Gram $H^\top H=1$, so that value alone cannot recover the typed allocation.

\begin{corollary}[Composite commit stability]
\label{cor:composite}
Let $V_{\parallel},E_{\parallel}\in\mathbb{R}^{r\times k}$, $V_{\perp},E_{\perp}\in\mathbb{R}^{(d-r)\times k}$, and $\ell,e_\ell\in\mathbb{R}^4$. For $j\in\{\parallel,\perp\}$, set $\widetilde V_j=V_j+E_j$ and $\widetilde\ell=\ell+e_\ell$. Let $\phi=\phi_{\textsc{Cask}}$ from Equation~\eqref{eq:cask-feature}, construct $\widetilde\phi$ analogously from the perturbed quantities, and define $\Delta\phi=\widetilde\phi-\phi$ and
\[
b_j=2\|V_j\|_2\|E_j\|_F+\|E_j\|_F^2.
\]
Then the complete CASK feature obeys
\begin{equation}
\|\Delta\phi\|_2\le\sqrt{b_{\parallel}^2+b_{\perp}^2+\|e_\ell\|_2^2}.
\end{equation}
Let $D\in\mathbb{R}^{46\times46}$ be the fixed positive diagonal training-set standardizer (constant coordinates use unit scale), let $\mu,\bar w\in\mathbb{R}^{46}$ and $\bar b\in\mathbb{R}$, and write the fitted logit as $F(\phi)=\bar w^\top D^{-1}(\phi-\mu)+\bar b$. Define $w_{\mathrm{eff}}=D^{-\top}\bar w$ and $t_\tau=\log(\tau/(1-\tau))$. The admission decision at any $\tau\in(0,1)$ is preserved whenever $|F(\phi)-t_\tau|$ exceeds $\|w_{\mathrm{eff}}\|_2$ times the right-hand side.
\end{corollary}
\begin{proof}
Apply Theorem~\ref{thm:stability} to each Gram block. Upper-triangle vectorization has Euclidean norm no larger than the Frobenius norm of the corresponding matrix, so the two block perturbations have norms at most $b_{\parallel}$ and $b_{\perp}$. Concatenation with the logit perturbation gives the displayed bound by the Pythagorean identity. Since $F(\widetilde\phi)-F(\phi)=w_{\mathrm{eff}}^\top\Delta\phi$, the decision claim follows from Theorem~\ref{thm:stability} with threshold $t_\tau$.
\end{proof}
Corollary~\ref{cor:composite} turns measured representation drift into a certificate for the deployed, standardized commit rule at either operating point.

\section{Exact diagnostic and selection protocol}
\label{app:diagnostics}
The implemented interface uses one user message containing ``Candidate memory:'', the complete framed record, and one of the following prompts. The Qwen chat template appends the generation prefix; no system message or decoded answer is used. We read the last-token output of the selected decoder block (the post-block residual returned by that block).
\begin{enumerate}[leftmargin=1.4em]
    \item Should the embedded proposition enter unconditional shared memory? Answer only Yes or No.
    \item Can this proposition be used as a premise about what actually happened? Answer exactly Yes or No.
    \item Does the speaker commit this proposition to the common ground? Answer exactly Yes or No.
    \item May a future agent retrieve this proposition as an established fact? Answer exactly Yes or No.
    \item Is this proposition owned by the actual world rather than a hypothetical or attitude context? Answer exactly Yes or No.
    \item Would storing this proposition without its reporting frame preserve its status? Answer exactly Yes or No.
\end{enumerate}
All six ordered states enter both Gram matrices. Here $z_t$ denotes the model's pre-softmax score for token $t$; the logit block is $(z_{\mathrm{Yes}}-z_{\mathrm{No}})$ for prompts 1--4, in that order, yielding $21+21+4=46$ features. We vectorize each symmetric Gram matrix by its upper triangle in row-major order: $(1,1),(1,2),\ldots,(1,6),(2,2),\ldots,(6,6)$. Prompts are evaluated without sampling and are truncated only beyond 256 input tokens.

Block 20 and rank 16 were fixed in a pilot before the final held-out evaluation; the 1.5B study transferred rank 16 and matched relative depth at block 16. A development-only $3\times3$ grid over blocks $\{18,20,22\}$ and ranks $\{8,16,32\}$ maps sensitivity around that choice on one fixed conversation split. Table~\ref{tab:sensitivity} reports attachment accuracy and the aligned frozen-gate intervention effect. All nine settings retain 97.7--100\% development accuracy and a positive causal effect (.305--.869), showing that the ownership route persists across nearby layers and ranks. Main test results use the prespecified block-20/rank-16 configuration.

\begin{table}[h]
\centering
\caption{Development sensitivity. Each cell is attachment accuracy (\%) / aligned $\Delta p$ under the ownership-subspace swap.}
\label{tab:sensitivity}
\small
\begin{tabular}{lccc}
\toprule
& Rank 8 & Rank 16 & Rank 32\\
\midrule
Block 18 & $99.2/.636$ & $100.0/.305$ & $98.4/.312$\\
Block 20 & $98.4/.788$ & $97.7/.716$ & $99.2/.683$\\
Block 22 & $100.0/.776$ & $100.0/.801$ & $100.0/.869$\\
\bottomrule
\end{tabular}
\end{table}

We retained each fold's training conversations and replaced the prompt bank, the reporting-frame bank, or both at test time; projector and gate fitting use the original banks. The five test-fold accuracies were $84.38,90.63,92.97,87.50,99.22$ for diagnostic paraphrases (mean $90.94$); $91.41,96.09,87.50,97.66,94.53$ for unseen reporting frames (mean $93.44$); and $65.63,67.97,73.44,68.75,75.00$ for the joint shift (mean $70.16$). The unchanged-data reference is $97.81$ and chance is $50\%$. Strong separate transfer shows that the ownership signal generalizes across prompt wording and reporting frames; the joint shift reveals their interaction.

\paragraph{Agent-trace protocol.}
We selected up to ten exact-span LoCoMo questions from each conversation, balanced across question categories where available, obtaining 75 eligible episodes. For each episode, Qwen2.5-3B-Instruct generated a tentative hypothesis, proposed search action, and unexecuted alternative; the runtime then executed conversation retrieval and returned either a successful observation or a failed-call event before the model produced a rejection and final conclusion. We removed event names and status fields before running the diagnostics. The runtime ledger labels successful observations actual; tentative, planned, alternative, rejected, and failed-call content scoped; and a final conclusion actual when supported by returned evidence. This produces 511 records: 92 actual and 419 scoped, grouped into five folds by source conversation. Each fold fits a trace-adapted head; the controlled-data head's direct-transfer score (31.66\% balanced accuracy) quantifies the domain shift.

\begin{table}[h]
\centering
\caption{Generated tool-agent traces (\%). BAcc is balanced accuracy; the 50\% hybrid averages five metadata-retention masks.}
\label{tab:traces}
\small
\begin{tabular}{lrrr}
\toprule
Commit policy & BAcc & Actual & Scoped\ $\uparrow$\\
\midrule
Output logits &82.29&81.52&83.05\\
Complement Gram+logits &86.37&85.87&86.87\\
TF--IDF &90.40&81.52&99.28\\
\method{} fallback &91.10&89.13&93.08\\
\method{} + 50\% provenance &\textbf{95.02}&\textbf{93.48}&\textbf{96.56}\\
Runtime provenance (100\%) &100.00&100.00&100.00\\
\bottomrule
\end{tabular}
\end{table}

The provenance hybrid replaces the learned score with the runtime label wherever metadata is retained. At 0\% retention it is exactly the \method{} fallback; at 50\% we average five independently sampled masks, obtaining 95.02\% balanced accuracy, 93.48\% actual admission, and 96.56\% scoped rejection; at 100\% it is the runtime-oracle upper bound. A per-episode meta-fact audit stores each scoped embedded proposition under an actual process wrapper (e.g., \textsc{Considered}) and verifies that process queries and world-state queries select the corresponding views.

\paragraph{Inference cost.}
Latency uses 200 records, 20 warm-up records, deterministic inference, and an A100-SXM4-80GB. At online batch size one, one diagnostic takes 24.83 ms median/26.42 ms p95 and 5.89 GiB peak memory. Six serial diagnostics take 148.76/153.76 ms, whereas batching the prompts takes 32.10/35.70 ms and 6.06 GiB. The resulting 1.29$\times$ online multiplier requires no custom prefix-cache optimization.

\section{Reproducibility summary}
All balanced learned systems use threshold 0.5. The source-isolated variants contain 4,860 records and 2,430 matched pairs each. Conversation clustering treats each of the ten conversations equally; paired bootstrap intervals resample conversations. Random-decomposition reporting uses one prespecified draw for clustered comparisons and 20 draws for the distributional control. The causal-routing test contains 146 held-out target cases grouped into 74 target scenarios. The answer-conflict task contains 972 records paired across 486 exact-span questions. Reader generations are deterministic for identical cached inputs, and model and threshold selection use training/development data.

\paragraph{Compute environment and existing assets.}
Experiments ran on an internal Linux server with NVIDIA A100-SXM4 80GB GPUs; GPU processes were single-device and at most two GPUs were active concurrently. The environment records Python 3.9, PyTorch 2.5.1, Transformers 4.55.0, NumPy 1.26.4, and scikit-learn 1.5.2. We obtained LoCoMo from its \href{https://github.com/snap-research/locomo}{official repository} under CC BY-NC 4.0. Qwen2.5-1.5B-Instruct is distributed under Apache 2.0, while Qwen2.5-3B-Instruct uses the Qwen Research License; both came from the official \href{https://huggingface.co/Qwen}{Qwen organization}. This non-commercial research use follows those terms, and the submission redistributes neither corpus nor model weights.

\newpage
\section*{NeurIPS Paper Checklist}

\begin{enumerate}

\item {\bf Claims}
    \item[] Question: Do the main claims made in the abstract and introduction accurately reflect the paper's contributions and scope?
    \item[] Answer: \answerYes{}
    \item[] Justification: The abstract and Sections 1 and 7 state the contribution, evidence boundary, and non-universal scope explicitly.

\item {\bf Limitations}
    \item[] Question: Does the paper discuss the limitations of the work performed by the authors?
    \item[] Answer: \answerYes{}
    \item[] Justification: Section 7 discusses dataset and model coverage, the binary ownership edge, compound distribution shift, trace calibration, retrieval choice, and deployment limits.

\item {\bf Theory assumptions and proofs}
    \item[] Question: For each theoretical result, does the paper provide the full set of assumptions and a complete (and correct) proof?
    \item[] Answer: \answerYes{}
    \item[] Justification: Section 4 states and proves three theorems and one proposition; Appendix A states and proves the composite stability corollary for the implemented feature.

    \item {\bf Experimental result reproducibility}
    \item[] Question: Does the paper fully disclose all the information needed to reproduce the main experimental results of the paper to the extent that it affects the main claims and/or conclusions of the paper (regardless of whether the code and data are provided or not)?
    \item[] Answer: \answerYes{}
    \item[] Justification: Sections 4--6 and Appendix B specify the representations, split isolation, sample counts, selection boundary, intervention controls, agent-trace construction, latency protocol, and clustered uncertainty used for the claims.

\item {\bf Open access to data and code}
    \item[] Question: Does the paper provide open access to the data and code, with sufficient instructions to faithfully reproduce the main experimental results, as described in supplemental material?
    \item[] Answer: \answerNo{}
    \item[] Justification: LoCoMo and Qwen2.5 are public and credited, but no anonymous implementation or derived-data artifact accompanies this submission.

\item {\bf Experimental setting/details}
    \item[] Question: Does the paper specify all the training and test details (e.g., data splits, hyperparameters, how they were chosen, type of optimizer) necessary to understand the results?
    \item[] Answer: \answerYes{}
    \item[] Justification: Sections 4--6 and Appendix B specify ranks, feature dimensions, grouped splitting, thresholds, development-only selection, agent-trace labels, and evaluation units needed to interpret the results.

\item {\bf Experiment statistical significance}
    \item[] Question: Does the paper report error bars suitably and correctly defined or other appropriate information about the statistical significance of the experiments?
    \item[] Answer: \answerYes{}
    \item[] Justification: Section 6 reports conversation-clustered bootstrap intervals and distinguishes them from the standard deviation across 20 random decompositions.

\item {\bf Experiments compute resources}
    \item[] Question: For each experiment, does the paper provide sufficient information on the computer resources (type of compute workers, memory, time of execution) needed to reproduce the experiments?
    \item[] Answer: \answerNo{}
    \item[] Justification: Appendix B reports the NVIDIA A100-SXM4 80GB hardware, single-device GPU processes, a two-GPU concurrency limit, and the software environment; exact per-run and aggregate wall-clock times were not retained.
    
\item {\bf Code of ethics}
    \item[] Question: Does the research conducted in the paper conform, in every respect, with the NeurIPS Code of Ethics \url{https://neurips.cc/public/EthicsGuidelines}?
    \item[] Answer: \answerYes{}
    \item[] Justification: The study uses released research assets, reports negative results and limitations, and makes no claims about protected attributes or human participants.

\item {\bf Broader impacts}
    \item[] Question: Does the paper discuss both potential positive societal impacts and negative societal impacts of the work performed?
    \item[] Answer: \answerYes{}
    \item[] Justification: Sections 1 and 7 discuss preventing unauthorized memory promotion and the risks of incorrect attachment, false rejection, and overgeneralization.
    
\item {\bf Safeguards}
    \item[] Question: Does the paper describe safeguards that have been put in place for responsible release of data or models that have a high risk for misuse (e.g., pre-trained language models, image generators, or scraped datasets)?
    \item[] Answer: \answerNA{}
    \item[] Justification: The work does not release a new pretrained model or high-risk scraped dataset.

\item {\bf Licenses for existing assets}
    \item[] Question: Are the creators or original owners of assets (e.g., code, data, models), used in the paper, properly credited and are the license and terms of use explicitly mentioned and properly respected?
    \item[] Answer: \answerYes{}
    \item[] Justification: Appendix B credits the official sources and identifies LoCoMo as CC BY-NC 4.0, Qwen2.5-1.5B-Instruct as Apache 2.0, and Qwen2.5-3B-Instruct under the Qwen Research License; the study uses them for non-commercial research without redistributing the assets.

\item {\bf New assets}
    \item[] Question: Are new assets introduced in the paper well documented and is the documentation provided alongside the assets?
    \item[] Answer: \answerNA{}
    \item[] Justification: The study constructs evaluation records for its experiments but does not introduce a separately released dataset or model artifact.

\item {\bf Crowdsourcing and research with human subjects}
    \item[] Question: For crowdsourcing experiments and research with human subjects, does the paper include the full text of instructions given to participants and screenshots, if applicable, as well as details about compensation (if any)? 
    \item[] Answer: \answerNA{}
    \item[] Justification: No crowdsourcing or new human-subject study was conducted.

\item {\bf Institutional review board (IRB) approvals or equivalent for research with human subjects}
    \item[] Question: Does the paper describe potential risks incurred by study participants, whether such risks were disclosed to the subjects, and whether Institutional Review Board (IRB) approvals (or an equivalent approval/review based on the requirements of your country or institution) were obtained?
    \item[] Answer: \answerNA{}
    \item[] Justification: No new research with human participants was conducted.

\item {\bf Declaration of LLM usage}
    \item[] Question: Does the paper describe the usage of LLMs if it is an important, original, or non-standard component of the core methods in this research? Note that if the LLM is used only for writing, editing, or formatting purposes and does \emph{not} impact the core methodology, scientific rigor, or originality of the research, declaration is not required.
    \item[] Answer: \answerYes{}
    \item[] Justification: Sections 4--6 identify the Qwen2.5-3B-Instruct and Qwen2.5-1.5B-Instruct backbones as frozen representation and diagnostic-response models and describe their role in the method.

\end{enumerate}

\end{document}